\documentclass[letterpaper, 10 pt, conference]{ieeeconf}

\usepackage{amssymb, amsmath, multirow, graphicx,subfigure, slashbox,algorithm,algorithmic, threeparttable}

\newtheorem{theorem}{Theorem}

\newtheorem{proposition}[theorem]{Proposition}

\newenvironment{remark}[1][Remark]{\begin{trivlist}
\item[\hskip \labelsep {\bfseries #1}]}{\end{trivlist}}

  \def\bi{\begin{itemize}}
\def\ei{\end{itemize}}
\def\be{\begin{enumerate}}
\def\ee{\end{enumerate}}

\def\0{\mbox{\boldmath $0$}}
\def\1{\mbox{\boldmath $1$}}
 
\def\r{r}
\def\O{\mathcal{O}}

\begin{document}
\title{Inverse Reinforcement Learning with Gaussian Process}

\author{Qifeng Qiao and Peter A. Beling \\
Department of Systems and Information Engineering\\
University of Virginia\\
Charlottesville, Virginia 22904\\
Email: qq2r, pb3a@virginia.edu}

\maketitle

\begin{abstract}
We present new algorithms for inverse reinforcement learning (IRL, or inverse optimal control) in convex optimization settings. We argue that finite-space IRL can be posed as a convex quadratic program under a Bayesian inference framework with the objective of maximum a posterior estimation.  To deal with problems in large or even infinite state space, we propose a Gaussian process model and use preference graphs to represent observations of decision trajectories.  Our method is distinguished from other approaches to IRL in that it makes no assumptions about the form of the reward function and yet it retains the promise of computationally manageable implementations for potential real-world applications.  In comparison with an establish algorithm on small-scale numerical problems, our method demonstrated better accuracy in apprenticeship learning and a more robust dependence on the number of observations.

\end{abstract}

\section{Introduction}
\label{introduction}
Imitation learning is a subfield of machine learning in which the objective is to learn to mimic human behavior solely through observation of the actions taken by the subject.  Technical approaches to imitation learning generally fall into two broad categories \cite{ratliff2009}.  One category contains behavioral cloning approaches that attempt to use supervised learning to predict actions directly from observations of features of the environment. The other category consists of IRL approaches, first introduced in \cite{andrew2000}, use training examples in the form of decision trajectories defined in terms of a Markov decision process (MDP) model of the underlying sequential decision task. IRL algorithms attempt to discover the reward function for the MDP solely on the basis of observations of a decision-maker's solution to that problem.  This approach is appealing  because knowledge of the reward function offers the promise that behavior can be predicted in domains unseen during the period of observation.   

A variety of approaches have been proposed for IRL. In early work, Ng and Russel \cite{andrew2000} advance the key idea  of choosing the reward function to maximize the difference between the optimal and suboptimal policies, under the assumption that the reward function can be approximated by a linear combination of basis functions. A principal motivation for considering IRL problems is the idea of apprenticeship learning, in which observations of state-action pairs are used to learn the policies followed by experts for the purpose of mimicking or cloning behavior.  By its nature, apprenticeship learning problems arise in situations where it is not possible or desirable to observe all state-action pairs for the decision maker's policy.    In recent approaches to apprenticeship learning,  partial policy observation is dealt with by searching mixed solutions in a  space of learned policies with the goal that the accumulative feature expectation is near that of the expert  \cite{abbeel2004, umar2008a}. In such approaches, the reward function is approximated by a linear combination of features, which in turn allows for linear approximation of value functions with consequent simplification of the learning problem.    In such methods, algorithm performance is strongly influenced by the modeler's choice of features.  Another algorithm for IRL is policy matching in which the loss function penalizing deviations from expert's policy is minimized by tuning the parameters of reward functions \cite{neu2007}.

The assumption that the reward function can be linearly approximated, which underlies a number of IRL approaches, may not be  reasonable for many problems of practical interest.  The
 ill-posed nature of the inverse learning problem also presents difficulties.  Multiple reward functions may yield the same optimal policy, and there may be multiple observations at a state given the true reward function.  To deal with these problems, we design algorithms that do not assume linear structure for reward function, but yet remain computationally efficient. In particular, we propose new IRL models and algorithms that assign a Gaussian prior on the reward function or treat the reward function as a Gaussian process.   This approach is similar in perspective to that Ramachandran and Eyal  \cite{deepak2007}, who view the state-action samples from the expert as the evidence that will be used to update a prior on the reward function, under a Bayesian framework.  Other approaches to IRL include game-theoretic methods \cite{umar2008b}  and algorithms derived from linearly-solvable stochastic optimal control \cite{dvijotham2010}.

The main contributions of our work are as follows. First, we model the reward function in a finite state space using a Bayesian framework with known Gaussian priors. We show that this problem is a convex quadratic program, and hence that it can be efficiently solved. Second, for the general case that allows noisy observation of incomplete policies, representation of the reward function is challenging and requires more computation. We show that Gaussian process is appropriate in that case. Our model constructs a preference graph in action space to represent the multiple observations at a state.  Even in cases where the state space is much larger than the number of observations, IRL via Gaussian processes has the promise of offering robust predictions and results that are relatively insensitive to number of observations.
 
It is worth mentioning here that the preference graph we use in IRL is based on an understanding of the agent's preferences over action space.  In the machine learning literature, there has been study of a learning scenario called learning label preference that focuses on finding the latent function that predicts preference relations among a finite set of labels. This scenario is a generalization of some standard problems, such as classification and label ranking \cite{Furnkranz2005}. Considering the latent function values as a Gaussian process, Chu and Ghahramani \cite{chu2005} observed that Bayesian framework is an efficient and competitive method for learning label preferences, and  they proposed a novel likelihood function to capture preference relations and the use of a Gaussian process model for learning label  preferences.  We also use Bayesian inference and build off several of the ideas in \cite{chu2005} and related work, but our method differs from label preference learning for classification and label ranking. Our input data depends on states and actions in the context of an MDP. Moreover, we are learning the reward that indirectly determines how actions are chosen during the sequential evolution of an MDP, while preference learning studies the latent functions preserving preferences. 

The rest of this paper is organized as follows: In Section \ref{section:preliminaries}, we introduce IRL preliminaries. In Sections \ref{section:convex} and  \ref{section:gpmodel}, we propose our principal models and algorithms. In Section \ref{section:experiments}, we describe the results of two small-scale numerical experiments.  Finally,  in Section \ref{section:conclusion}, we offer some concluding remarks.

\section{Preliminaries}
\label{section:preliminaries}

A finite-state, infinite horizon {\em Markov decision process (MDP)} is defined as a tuple $\textsl{M}=(\mathcal{S},\mathcal{A},\mathcal{P},\gamma, \r)$, where 
$\mathcal{S} = \left\{s_{1}, s_{2},\cdots, s_{n}\right\}$ is a set of $n$ states; $\mathcal{A}=\left\{a_{1}, a_{2},\cdots, a_{m}\right\}$ is a set of $m$ actions; $\mathcal{P}=\left\{P_{a_{j}}\right\}^{m}_{j=1}$ is a set of state transition probabilities; $\gamma$ is a discount factor; and $r$ is the reward function which can be written as $r(s,a)$, if we define it as depending on state $s$ and action $a$. For any $a\in A$ and $P_{a}$ is a $n\times n$ matrix, each row of which, denoted as $P_{as}$, is the transition probabilities upon taking action $a$ in state $s$.

Consider a decision maker who selects actions according to a {\em policy} $\pi: S\rightarrow A$ that maps states to actions.  Define the {\em value function} at state $s$ with respect to policy $\pi$ to be $V^{\pi}(s)=E[\sum_{t=0}^{\infty}\gamma^{t}r(s^{t},\pi(s^{t}))|\pi]$, where the expectation is over the distribution of the state sequence $\left\{s^{0}, s^{1}, \dots \right\}$ given policy $\pi$, where superscripts index time.   A decision maker who aims to maximize expected reward will, at every state $s$, choose the action that maximizes $V^{\pi}(s)$. Similarly, define the  $Q$-factor for state $s$ and action $a$ under policy $\pi$, $Q^{\pi}(s,a)$, to be the expected return from state $s$, taking action $a$ and thereafter following policy $\pi$. 
Given a policy $\pi$, $\forall s\in S, a\in \mathcal{A}$, $V^{\pi}(s)$ and $Q^{\pi}(s,a)$ satisfy
\begin{eqnarray}
V^{\pi}(s) &=&  r(s,\pi(s)) +\gamma\sum_{s'}P_{\pi(s)s}(s')V^{\pi}(s') \nonumber\\
\ Q^{\pi}(s,a) &=& r(s,a) + \gamma \sum_{s'}P_{as}(s')V^{\pi}(s')\nonumber
\label{bellman}
\end{eqnarray}
The well-known Bellman optimality conditions state that $\pi$ is optimal if and only if, $\forall s\in S$, we have $ \pi(s)\in {\rm arg} \max_{a\in \mathcal{A}}Q^{\pi}(s,a)$ \cite{bellman1957}.

Given an MDP $M=(\mathcal{S},\mathcal{A},\mathcal{P},\gamma,\ r)$, let us define the {\em inverse Markov decision process (IMDP)} $M_{I} =(\mathcal{S},\mathcal{A},\mathcal{P},\gamma, \mathcal{O})$.  The process $M_I$ includes the states, actions, and dynamics of $M$, but lacks a specification of the reward vector, $\r$.   By way of compensation, $M_I$ includes a set of observations $\mathcal{O}$ that consists of state-action pairs generated through the observation of a decision maker.   We can define the {\em inverse reinforcement learning} (IRL) problem associated with $M_I=(\mathcal{S},\mathcal{A},\mathcal{P},\gamma, \mathcal{O})$ to be that of finding the reward function $r$ such that the observations $\O$ could have come from an optimal policy for $M=(\mathcal{S},\mathcal{A},\mathcal{P},\gamma, \r)$.   The IRL problem is, in general, highly underspecified, which has led researchers to consider various models for restricting the set of reward vectors under consideration.  In a seminal consideration of IMDPs and associated IRL problems, Ng and Russel \cite{andrew2000} observe that, by the optimality equations, the only reward vectors consistent with an optimal policy $\pi$ are those that satisfy the set of inequalities 
\begin{eqnarray}
(P_{\pi}-P_{a})(I_{n}-\gamma P_{\pi})^{-1}\textbf{r} &\geq &\0, \quad \forall a \in \mathcal{A},  
\label{ieq:spi}
\end{eqnarray}
where $P_{\pi}$ is the transition probability matrix relating to observed policy $\pi$, $P_{a}$ denotes the transition probability matrix for other actions, $I_{n}$ is a $n\times n$ identity matrix, and $\textbf{r}$ is a reward vector that depends only on state.  Note that the trivial solution $\textbf{r}=0$ satisfies these constraints, which highlights the underspecified nature of the problem and the need for reward selection mechanisms.   Ng and Russel \cite{andrew2000} choose the reward function to maximize the difference between the optimal and suboptimal policies, which can be done using a linear programming formulation.  In the sections that follow, we propose the idea of selecting reward on the basis of Maximum a posterior (MAP) estimation in a Bayesian framework.

\section{Bayesian IRL with Gaussian Distribution}
\label{section:convex}

Suppose that we have a prior distribution $p(\r$) for the rewards in an IMDP $M_{I}$, along with a likelihood function $p(\O|\r)$.  Then we can define the associated Bayesian IRL problem to be that of finding the MAP estimate of $\r$.   In this section we consider this problem for priors with a Gaussian distribution, showing that the MAP estimation problem can be formulated as a convex optimization problem.  We assume all the states, value functions, and transition probabilities can be stored in the memory of a computer. 

Specifically, let $\textbf{r} \in \Re^{n}$ be a random vector only depending on state. The entry $\textbf{r}(s_{i})$ denotes the reward at i-th state. We assign a Gaussian prior on the $\textbf{r}$: $\textbf{r}\sim \mathcal{N}(\mu_{r}, \Sigma_{r})$. This is a subjective distribution; before anything is known about optimal policies for the MDP, the learner has characterized a prior belief by $\mu_{r}$ with confidence by $\Sigma_{r}$. 

\begin{figure}[!tb]
	\centering
	{\includegraphics[width =3.6in]{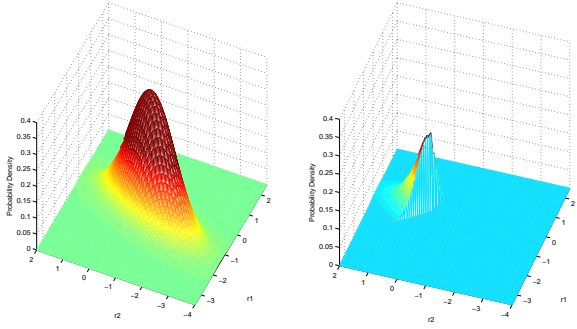}}
	\caption{An example showing the Bayesian IRL given full observation of the decision maker's policy.}
	\label{fig:gpraw}
\end{figure}


One can envision two principal types of experiments for collecting a set of observations $\mathcal{O}$:
\begin{enumerate}
	\item \textsl{Decision Mapping}: the observations are obtained by finding a mapping between state and action; e.g., we ask the expert which action he, she, or it would choose at state $s$, and then repeat the process.  Ultimately, we will have a set of independent state-action pairs, $\mathcal{O}_{1}=\left\{(s^{h},a^{h})\right\}^{t}_{h=1}$.  
	\item \textsl{Decision Trajectory}: Given an initial state, we simulate the decision problem and record the history of the expert's behavior, $\mathcal{O}_{2}=\left\{s^{1},a^{1}, s^{2}, a^{2},\cdots,s^{t},a^{t}\right\}$.
\end{enumerate}

Formally, we define an experiment $E$ to be a triple $(O, \textbf{r}, \left\{p(\mathcal{O}|\textbf{r})\right\})$, where $O$ is a random vector with probability mass function $p(\mathcal{O}|\textbf{r})$ for some $\textbf{r}$ in the function space. Given what experiment E was performed and a particular observation of $\mathcal{O}$, the experimenter is able to make inference and draw some evidence about $\textbf{r}$ arising from $E$ and $\mathcal{O}$. This evidence we denote by $Ev(E,\mathcal{O})$. Consider observations made using decision mapping $\mathcal{O}_{1}$ and decision trajectory $\mathcal{O}_{2}$, with corresponding experiments  $E_{1} = (O_{1}, \textbf{r}, \left\{p(\mathcal{O}_{1}|\textbf{r})\right\})$ and $E_{2} = (O_{2}, \textbf{r}, \left\{p(\mathcal{O}_{2}|\textbf{r})\right\})$. We would like to show that $Ev(E_{1},\mathcal{O}_{1})= Ev(E_{2},\mathcal{O}_{2})$, if the states in $\mathcal{O}_{1}$ and $\mathcal{O}_{2}$ are the same. This fact implies that inference conclusions drawn from $\mathcal{O}_{1}$ and $\mathcal{O}_{2}$ should be identical.

Making use of independence of state-action pairs in decision mapping, we calculate the joint probability density as
\begin{eqnarray}
p(\mathcal{O}_{1}|\textbf{r}) = \prod_{h=1}^{t} p(s^{h},a^{h}|\textbf{r}) = \prod_{h=1}^{t} p(s^{h})p(a^{h}|s^{h},\textbf{r}).\nonumber
\label{eq:obs1}
\end{eqnarray}
Considering Markov transition in decision trajectory, we write the joint probability density as
\begin{eqnarray}
p(\mathcal{O}_{2}|\textbf{r})= p(s^{1}) p(a^{1}|s^{1},\textbf{r})\prod_{h=2}^{t} p(s^{h}|s^{h-1},a^{h-1})p(a^{h}|s^{h},\textbf{r}).\nonumber
\label{eq:obs2}
\end{eqnarray}
Finally, we get $p(\mathcal{O}_{1}|\textbf{r}) = c(\mathcal{O}_{1}, \mathcal{O}_{2})p(\mathcal{O}_{2}|\textbf{r})$,
where $c(\mathcal{O}_{1}, \mathcal{O}_{2})$ is a constant. The above equation implies an equivalence of evidence for inference of $\textbf{r}$  between the use of  a decision map or  a decision trajectory.

To simplify computation, we eliminate the elements in likelihood function $p(\mathcal{O}|\textbf{r})$ that do not contain $\textbf{r}$, which yields $p(\mathcal{O}|\textbf{r}) = \prod_{h=1}^{t} p(a^{h}|s^{h},\textbf{r})$. Further, we model $p(a^{h}|s^{h},\textbf{r})$ by
	\begin{eqnarray}
	p(a^{h}|s^{h}, \textbf{r}) = 
	\begin{cases}
	1, {\rm if}\ Q(s^{h},a^{h}) \geq Q(s^{h},a),\ \forall a\in \mathcal{A}\\
	0, {\rm otherwise.}
	\end{cases}
	\end{eqnarray}	
This form for the likelihood function is based on the assumption each observed action is an optimal choice on the part of the expert. Note that the set of reward values that make $p(a^{h}|s^{h},\textbf{r})$ equal to one is given by Eq. \ref{ieq:spi}. 

\begin{proposition}
\label{proposition:qcpirl}
Assume a countable state and control space and a stationary policy. Then IRL using Bayesian MAP inference is a quadratic convex programming problem.
\end{proposition}
\begin{proof}
By Bayes rule, the posterior distribution of reward
\begin{eqnarray}
p(\textbf{r}|\mathcal{O}) = \frac{1}{(2\pi)^{n/2}|\Sigma_{r}|^{1/2}}\exp\left(-\frac{1}{2}(\textbf{r}-\mu_{r})^{T}\Sigma_{r}^{-1}(\textbf{r}-\mu_{r})\right).\nonumber 
\end{eqnarray}
This posterior probability $p(\textbf{r}|\mathcal{O})$ quantifies the evidence that $\textbf{r}$ is the reward for the observations in $\mathcal{O}$.  Using Eq. \ref{ieq:spi}, we formulate the IRL problem as
\begin{eqnarray}
&&\min\limits_{r}\ \frac{1}{2}(\textbf{r}-\mu_{r})^{T}\Sigma_{r}^{-1}(\textbf{r}-\mu_{r}) \nonumber\\
&&{\rm s.t.}\ \ (P_{a^{*}}-P_{a})(I_{n}-\gamma P_{a^{*}})^{-1}\textbf{r}> 0, \ \forall a \in \mathcal{A} \\
&&\ \ \ \ \ \textbf{r}_{\rm min}<\textbf{r}<\textbf{r}_{\rm max}\nonumber
\label{prob:cpirl}
\end{eqnarray}
Since the objective is convex quadratic and constraints are affine, Problem 3 is a convex quadratic program. 
\end{proof}

Fig. \ref{fig:gpraw} shows a Gaussian prior on reward and its posterior after truncation by accounting for the linear constraints on reward implied by observation $\mathcal{O}$.  Note the shift in mode.   

The development above assumes the availability of a complete set of observations, giving the optimal action at every state. If necessary, it may be possible to expand  observations of partial policies to fit the framework.  A naive approach would be to state transition probabilities averaged over all possible actions at unobserved states.

\section{Gaussian processes for generalized IRL}
\label{section:gpmodel}

 In this section, we introduce a Gaussian process IRL model.  Our model involves the construction of a preference graph, defined below, that is used to record the actions of the expert under observation.  The choice of one action over the others at any given state will be governed by Q-function values, if the expert acts optimally.  Hence, these values may be used to define preference relations among actions.

\textsl{Definition 1} At state $s_{i}\in \mathcal{S}$, $\forall \hat{a}, \check{a} \in \mathcal{A}$, we define the \textsl{preference relation} as: if $Q(s_{i},\hat{a})\geq Q(s_{i},\check{a})$, the action $\hat{a}$ is weakly preferred to $\check{a}$, denoted as $\hat{a}\succeq_{s_{i}} \check{a}$; strictly preferred, denoted as $\hat{a}\succ_{s_{i}} \check{a}$,if and only if $Q(s_{i},\hat{a})> Q(s_{i},\check{a})$; $\hat{a}$ is equivalent to $\check{a}$, denoted as $\hat{a}\sim_{s_{i}} \check{a}$, if and only if $\hat{a}\succeq_{s_{i}} \check{a}$ and $\check{a}\succeq_{s_{i}} \hat{a}$.

\textsl{Definition 2} A \textsl{preference graph} over action space is a directed graph showing preference relations among the countable actions at a given state. At state $s_{i}$, a preference graph $\epsilon_{i}$ consists of the node set $\mathcal{V}_{i}$ and edge set $\mathcal{E}_{i}$. Each node represents an action in $\mathcal{A}$. Define a one-to-one mapping $\varphi: \mathcal{V}_{i}\rightarrow\mathcal{A}$. Each edge indicates the preference relation between two nodes.

Suppose we are given a dataset of observations, denoted as $\mathcal{O} =\left\{\mathcal{S}, \mathcal{G}\right\}=\left\{s_{i}, \epsilon_{i}\right\}^{\hat{n}}_{i=1}$. Each pair $(s_{i}, \epsilon_{i})$ consists of two components: one is the input $s_{i}$ that is a feature vector constructed by a mapping $\phi:\mathcal{S}\rightarrow[0,1]^{d}$; the other, denoted as $\epsilon_{i}=(\mathcal{V}_{i},\mathcal{E}_{i})$, is a two layer preference graph over actions observed at $s_{i}$. As shown in Figure \ref{fig:preferencegraph}, the node set $\mathcal{V}_{i}$ can be divided into two subsets: a set of nodes in the top layer to represent optimal actions, denoted as $\mathcal{V}^{+}_{i}$; a set of nodes in the bottom layer to represent other actions, denoted as $\mathcal{V}^{-}_{i}$. The graph $\epsilon_{i} = \left\{(u \rightarrow v)_{l=1}^{n_{i}}, u\in \mathcal{V}^{+}_{i},\ v\in \mathcal{V}^{-}_{i} \right\}\cup \left\{(u\leftrightarrow v)_{k=1}^{m_{i}},\ u,v\in \mathcal{V}^{+}_{i}\right\}$, where $n_{i}$ is the number of edges denoting strict preference relations and $m_{i}$ is the number of edges denoting equivalent relations.
\begin{figure}[h]
	\centering
		\subfigure[]{\includegraphics[width =1.3in]{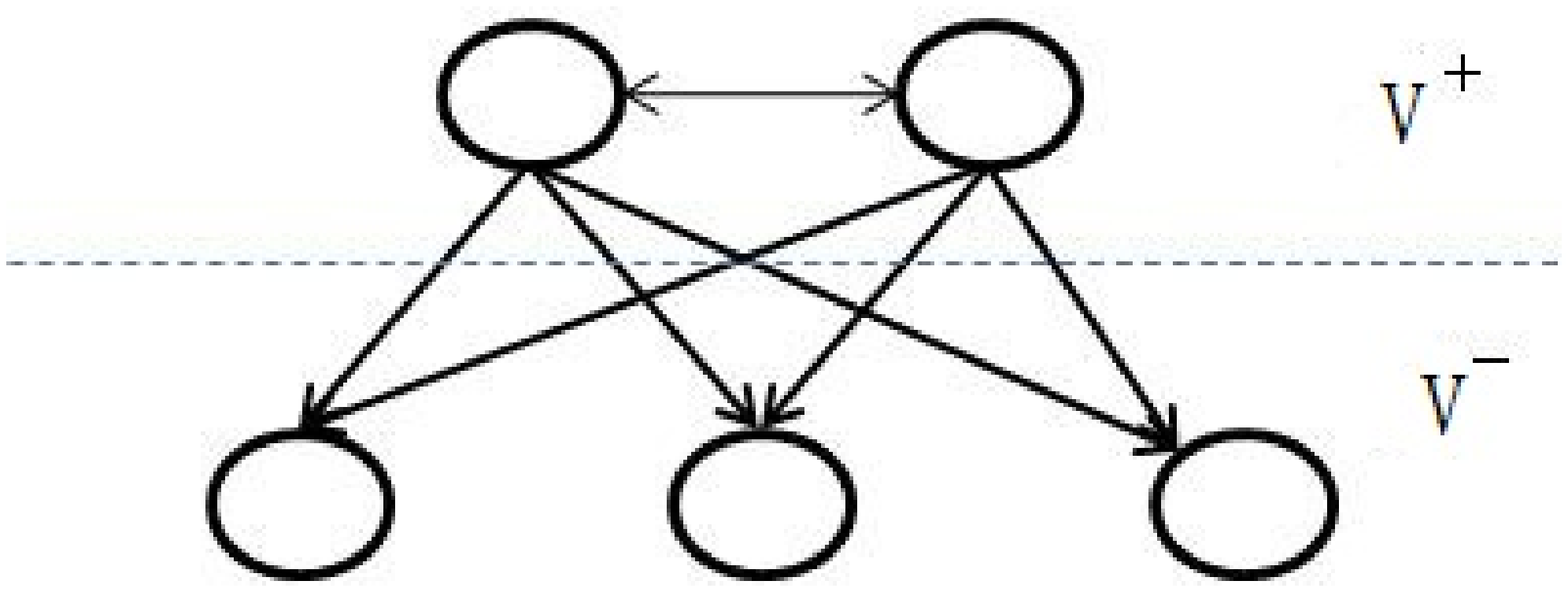}}
		\hfill
		\subfigure[]{\includegraphics[width =1.3in]{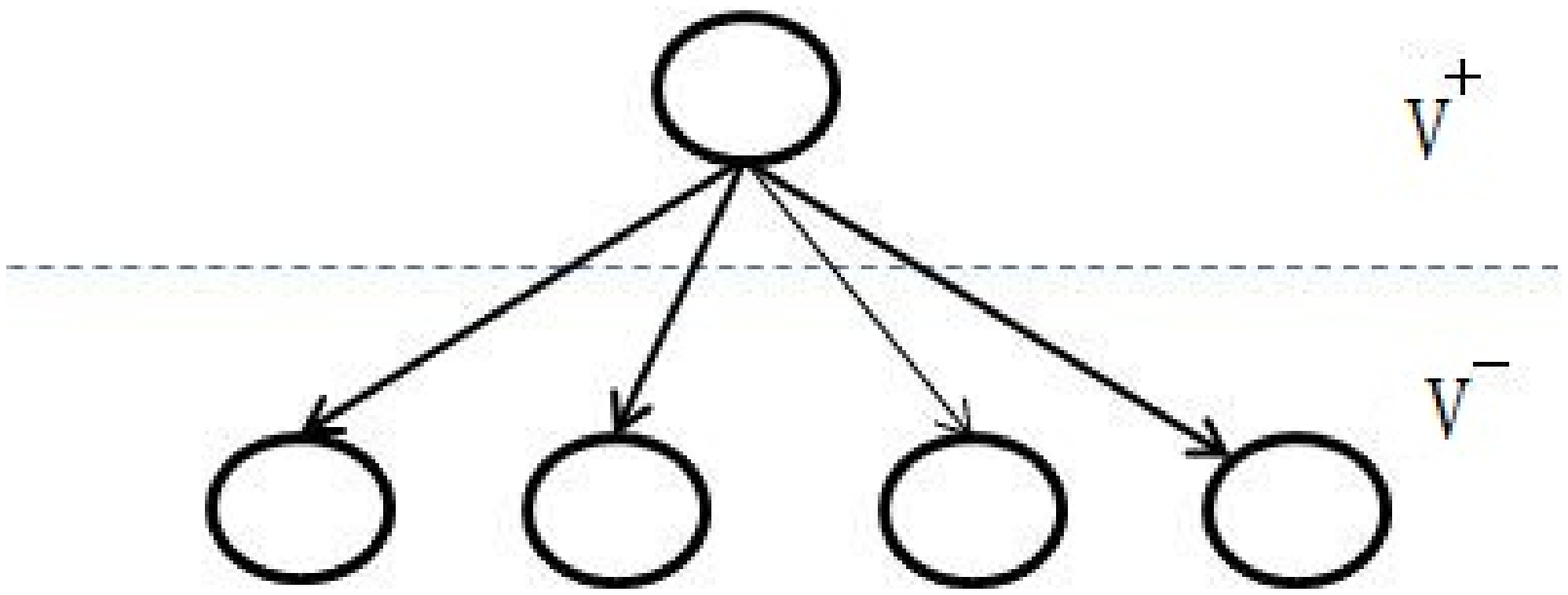}}
	\caption{Examples of preference graph}
	\label{fig:preferencegraph}
\end{figure}
Consider action's influence on the reward function. Here we define $\textbf{r}$ as follows.
\begin{eqnarray}
\textbf{r}&=&(\underbrace{\textbf{r}_{a_{1}}(s_{1}),...,\textbf{r}_{a_{1}}(s_{\hat{n}})}, \ldots ,\underbrace{\textbf{r}_{a_{m}}(s_{1}), \ldots ,\textbf{r}_{a_{m}}(s_{\hat{n}})})\nonumber\\
&=& (\ \ \ \ \ \ \ \ \ \ \ \textbf{r}_{a_{1}},\ \ \ \ \ \ \ \cdots,\ \ \ \ \ \ \ \ \ \ \ \ \ \textbf{r}_{a_{m}})
\label{defonf}
\end{eqnarray}
 where $\textbf{r}_{a_{j}}, \forall j\in\left\{1,2,\cdots,m\right\}$, denotes the reward only associated with j-th action. Given $\textbf{r}$, a ranking function can be naturally formulated as arrangement of the nodes in sorting of the values of Q-functions. We write the ranking function with respect to a node $u$ at state $s$ as $Q(s,\varphi(u))$. 

\subsection{Bayesian inference}
Below we describe our models for prior information, likelihood functions, and inference.

\subsubsection{Gaussian prior}
\label{prior}
Consider $\textbf{r}_{a_{j}}$ as a stochastic process. Then $\textbf{r}_{a_{j}}$ is a Gaussian process if, for any $\left\{s_{1},\cdots,s_{\hat{n}}\right\} \in \mathcal {S}$, the random variables $\left\{\textbf{r}_{a_{j}}(s_{1}), \cdots, \textbf{r}_{a_{j}}(s_{\hat{n}})\right\}$ are normally distributed. We denote by $k_{a_{j}}(s_{c}, s_{d})$ the function generating the value of entry $(c,d)$ for covariance matrix $K_{a_{j}}$, which leads to $\textbf{r}_{a_{j}}\sim N(0,K_{a_{j}})$. 
Then the joint prior probability of the reward is a product of multivariate Gaussian, namely 
$p(\textbf{r}|\mathcal{S})=\prod\nolimits^{m}_{j=1}p(\textbf{r}_{a_{j}}|\mathcal{S})\,\ {\rm and}\ \textbf{r}\sim N(0, K)$. Thus $\textbf{r}$ is completely specified by the positive definite covariance matrix $K$. As we assume the $m$ latent processes are uncorrelated, the covariance matrix $K$ is block diagonal in the covariance matrices $\left\{K_{1},...,K_{m}\right\}$. In practice, we use  a squared exponential kernel function, written as $k_{a_{j}}(s_{c},s_{d}) = e^{\frac{1}{2}(s_{c}-s_{d})^{T}M_{a}(s_{c}-s_{d})}+\sigma^{2}_{a_{j}}\delta(s_{c},s_{d})$
where  $M_{a_{j}}=\kappa_{a_{j}} I_{\hat{n}}$ and $I_{\hat{n}}$ is an identity matrix of size $\hat{n}$. The function $\delta(.)$ is the Kronecker delta. 

\subsubsection{Likelihood}
\label{likelihood}
Given an edge $u \rightarrow v$, we adopt a variant of the likelihood function proposed by Chu and Ghahramani in \cite{chu2005} to capture the preference relation in that edge. Specifically, 
\begin{eqnarray}
p_{\rm ideal}(u\rightarrow v|\textbf{r}_{\varphi(u)}(s), \textbf{r}_{\varphi(v)}(s)) \nonumber\\
= 
\begin{cases}
1\ \ if\ Q(s, \varphi(u))> Q(s, \varphi(v))\\
0\ \ {\rm otherwise,}
\end{cases}
\end{eqnarray}
where $u$ and $v$ are two nodes in the preference graph.  By Definition 2, these nodes can be mapped to two actions $\varphi(u)$ and $\varphi(v)$ in space $\mathcal{A}$. We write the Q-function as,
\begin{eqnarray}
Q(s,a)  = \textbf{r}_{a}(s)+\gamma \hat{P}_{as}(I_{\hat{n}}-\gamma \hat{P}_{a^{*}})^{-1}\hat{I} \textbf{r}
\end{eqnarray} 
where $\hat{P}_{as}$ and $\hat{P}_{a^{*}}$ are transition probabilities for the observed $\hat{n}$ states, and $\hat{I}$ is a matrix with $\hat{n}$ rows and $\hat{n}\times m$ columns. The production of $\hat{I}$ and $\textbf{r}$ is a $\hat{n}\times 1$ vector containing the reward for taking the optimal action at each state. After assuming that the latent functions are contaminated with Gaussian noise that has zero mean and unknown variance $\sigma^{2}$ \cite{chu2005}, the likelihood function for $l$-th strict preference edge in graph $\epsilon_{i}$ becomes 
\begin{eqnarray}
&&p(u_{l}\rightarrow v_{l})|\textbf{r}_{\varphi(u_{l})}(s_{i})+\delta_{u_{l}},\textbf{r}_{\varphi(v_{l})}(s_{i})+\delta_{v_{l}}) \nonumber\\
&&= \int\int p_{\rm ideal}(\varphi(u_{l})\succ \varphi(v_{l})|\textbf{r}_{\varphi(u_{l})}(s_{i}),\textbf{r}_{\varphi(v_{l})}(s_{i}))\nonumber\\
&&\mathbf{N}(\delta_{u}, 0, \sigma^{2})\mathbf{N}(\delta_{v}, 0, \sigma^{2})d\delta_{u}d\delta_{v}=\Phi(z^{l}_{i})
\end{eqnarray}
where $z^{l}_{i}=\frac{Q(s_{i},\varphi(u_{l}))-Q(s_{i},\varphi(v_{l}))}{\sqrt{2}\sigma}$, $\mathbf{N}(\delta_{u}, 0, \sigma^{2})$ denotes a Gaussian distribution for $\delta_{u}$, and $\Phi(z) = \int^{z}_{-\infty}\mathbf{N}(\gamma,0,1)d\gamma$. The l-th edge $(u_{l}\rightarrow v_{l})$ in preference graph $\epsilon_{i}$ denotes the strict preference relation $\varphi(u_{l}) \succ \varphi(v_{l})$. Consequently, we have $p(\varphi(u_{l})\succ_{s_{i}} \varphi(v_{l})|\textbf{r}) = \Phi(z^{l}_{i})$. With a two-layer preference graph, we are only interested in the directed edges between two layers as well as the equivalent relation in the top layer. We propose a new likelihood function for the $k$-th equivalent preference edge as follows,
\begin{eqnarray}
p(u_{k}\leftrightarrow v_{k} |\textbf{r})\propto e^{-\frac{1}{2}(Q(s_{i},\varphi(u_{k}))-Q(s_{i},\varphi(v_{k})))^{2}}
\label{eq:likelihoodeqal}
\end{eqnarray}
where $u_{k},v_{k}\in V^{+}$ and the k-th edge $(u_{k} \leftrightarrow v_{k})$ denotes the equivalent relation $\varphi(u_{k})\sim_{s_{i}} \varphi(v_{k})$. We have $p(\varphi(u_{k})\sim_{s_{i}} \varphi(v_{k})|\textbf{r}) = p(u_{k}\leftrightarrow v_{k} |\textbf{r})$ that is shown in Eq.\ref{eq:likelihoodeqal}.
Then we compute the likelihood function for all observed preference graphs using the following equation, 
\begin{eqnarray}
p({\mathcal{G}}|\mathcal{S},\textbf{r}, \theta) = \prod\limits^{\hat{n}}_{i=1} p(\epsilon_{i}|s_{i},\textbf{r})= \prod\limits^{\hat{n}}_{i=1}\prod\limits^{n_{i}}_{l=1}\Phi(z^{l}_{i})\nonumber\\
\exp(\sum_{i=1}^{\hat{n}}\sum_{k=1}^{m_{i}} -\frac{1}{2} (Q(s_{i},\varphi(u_{k}))-Q(s_{i},\varphi(v_{k})))^{2}). 
\end{eqnarray}

We put all the unknown parameters into a hyper-parameter vector $\theta = \left\{\kappa_{a_{j}}, \sigma_{a_{j}}, \sigma\right\}$, and then adjust the hyper-parameters on the basis of  maximum a posterior estimation.

\subsubsection{Posterior inference}
Here we adopt a hierarchical model. At the lowest level are function values encoded as a parameter vector $\textbf{r}$. At the top level are hyper-parameters in $\theta$ controlling the distribution of the parameters at the bottom level. Inference takes place one level at a time. At the bottom level, the posterior over function values are given by Bayes' rule as $p(\textbf{r}|\mathcal{S},\mathcal{G}, \theta)= p(\mathcal{G}|\mathcal{S},\theta,\textbf{r})p(\textbf{r}|\mathcal{S},\theta)/p(\mathcal{G}|\mathcal{S},\theta)$.

The posterior combines the information from the prior and the data, which reflects the updated belief about $\textbf{r}$ after observing the decisions. By Eq. \ref{prob:cpirl}, our task is to minimize the negative log posterior equation $U(\textbf{r})$, which is
\begin{eqnarray}
U(\textbf{r}) &=& \frac{1}{2}\sum\limits^{m}_{j=1}\textbf{r}_{a_{j}}^{T}K^{-1}_{a_{j}}\textbf{r}_{a_{j}} + \sum_{i=1}^{\hat{n}}\sum_{k=1}^{m_{i}}\frac{1}{2}(\sum^{m}_{j=1}\rho_{a_{j}}^{ik}\textbf{r}_{a_{j}})^{2}\nonumber\\
&-& \sum^{\hat{n}}_{i=1}\sum^{n_{i}}_{l=1}\ln \Phi(z_{i}^{l}).
\label{marginopt}
\end{eqnarray}
Given the k-th equivalent relation $\varphi(u_{k})\sim\varphi(v_{k})$, let $\mathbf{\Delta}_{k} \triangleq \gamma(\hat{P}_{\phi(u_{k})s_{i}}-\hat{P}_{\phi(v_{k})s_{i}})(I_{\hat{n}}-\gamma \hat{P}_{a^{*}})^{-1}$, then we have
\begin{eqnarray}
\rho^{ik}_{a_{j}} = \mathbf{\alpha}_{i}[\mathbf{1}(a_{j}=\phi(u_{k}))-\mathbf{1}(a_{j}=\phi(v_{k}))]+\mathbf{\Delta}_{k}\hat{I}_{a_{j}}\nonumber
\end{eqnarray}
where $\hat{I}_{a_{j}}$ is a block matrix of $\hat{I} = [\hat{I}_{a_{1}},\hat{I}_{a_{2}},\cdots,\hat{I}_{a_{m}}]$ and $\mathbf{\alpha}_{i}$ is a $1\times \hat{n}$ vector whose entry $\mathbf{\alpha}_{i}(i) = 1$, and $ \mathbf{\alpha}_{i}(-i) = 0$. The notation $\mathbf{1}(.)$ is an indicator function.
\begin{remark}
Minimizing Eq.\ref{marginopt} is a convex optimization problem.  The proof can be found in our supplemental report  \cite{qifeng2010}.
\end{remark}
At the minimum of $U(\textbf{r})$ we have
\begin{eqnarray}
\frac{\partial U}{\partial \textbf{r}_{a_{j}}}= 0 \Rightarrow \hat{\textbf{r}}_{a_{j}}=K_{a_{j}}(\nabla \log P(\mathcal{G}|\mathcal{S},\hat{\textbf{r}}, \theta)),
\label{mapf}
\end{eqnarray}
where $\hat{\textbf{r}} = (\hat{\textbf{r}}_{1},\cdots, \hat{\textbf{r}}_{a_{j}},\cdots,\hat{\textbf{r}}_{m})$. In Eq.\ref{mapf}, we can use Newton's method to find the maximum of $U$ with the iteration
\begin{eqnarray}
\textbf{r}_{a_{j}}^{{\rm new}} = \textbf{r}_{a_{j}} - (\frac{\partial^{2}U}{\partial \textbf{r}_{a_{j}} \partial \textbf{r}_{a_{j}}})^{-1}\frac{\partial U}{\partial \textbf{r}_{a_{j}}}.\nonumber
\end{eqnarray}

\subsection{Model selection}
\label{modelselection}
Model selection is the process of choosing a covariance function for a Gaussian process. The process can be considered to be training of a Gaussian process \cite{carl2006}. At the top level, we can optimize the hyper-parameters by maximizing the posterior over these hyper-parameters. The posterior $p(\theta|\mathcal{G}, \mathcal{S})$ is given by $p(\theta|\mathcal{G}, \mathcal{S}) = p(\mathcal{G}|\mathcal{S}, \theta)p(\theta)/p(\mathcal{G}|\mathcal{S})$, where the normalizing constant can be omitted for simplifying the optimization problem. If the prior distribution of hyper-parameters has no population basis, we assign the non-informative prior density to $\theta$. Optimization over $\theta$ becomes the problem of maximizing the marginal likelihood $p(\mathcal{G}|\mathcal{S}, \theta)$. We approximate the integral
of the marginal likelihood $p(\mathcal{G}|\mathcal{S}, \theta)$ using a Laplace approximation local expansion around the maximum, which is written as
\begin{eqnarray}
p(\mathcal{G}|\mathcal{S},\theta) \approx p(\mathcal{G}|\mathcal{S},\hat{\textbf{r}},\theta)\times p(\hat{\textbf{r}}|\theta)\delta_{\textbf{r}|\mathcal{S}}.
\label{laplace}
\end{eqnarray}
where $\delta_{\textbf{r}|\mathcal{S}}=|-\nabla\nabla\ln P(\textbf{r}|\mathcal{G},\mathcal{S},\theta)|^{-\frac{1}{2}}$ is the posterior uncertainty in $\textbf{r}$, which is known as the Occam factor, automatically incorporating a trade-off between model fit and model complexity. As the number of data increases, the approximation is expected to become increasingly accurate. The marginal likelihood can be further written as
\begin{eqnarray}
&&\log p(\mathcal{G}|\mathcal{S},\theta) = -U(\hat{\textbf{r}})-\frac{1}{2}\log\left|I_{\hat{n}}+K\Pi\right|, 
\label{log_hyper}
\end{eqnarray}
where $\hat{\textbf{r}}$ is the MAP estimation in Eq.\ref{mapf} and $\Pi$ is the second derivative matrix of the sum of the second and third part in Eq. \ref{marginopt}. Now we can find the optimal hyper-parameters by maximizing Eq.\ref{log_hyper}. 

\subsection{Posterior predictive reward}
When the observed state-action pairs are limited, e.g. in the large state space or infinite state space, how to predict the reward at new state is desirable. Our IRL with Gaussian process provides a probabilistic model to predict reward on new coming state $s^{*}$, which is a Gaussian model $p(\textbf{r}^{*}|\mathcal{G},\mathcal{S},s^{*},\theta)$ with the following mean function
\begin{eqnarray}
E(\textbf{r}_{a_{j}}^{*}|\mathcal{G}, \mathcal{S}, s^{*},\theta)  
&=& k_{a_{j}}(S, s^{*})^{T}(K_{a_{j}}+\sigma^{2}I_{\hat{n}})^{-1}\hat{\textbf{r}}_{a_{j}}\nonumber
\end{eqnarray}
and covariance function
\begin{eqnarray}
&&{\rm cov}(\textbf{r}_{a_{j}}^{*}|\mathcal{G}, \mathcal{S}, s^{*},\theta) = k_{a_{j}}(s^{*}, s^{*})\nonumber\\
&&- k_{a_{j}}(S, s^{*})^{T}(K_{a_{j}}+\sigma^{2}I_{\hat{n}})^{-1}k_{a_{j}}(S, s^{*}),\nonumber
\label{eq:posterior_predict}
\end{eqnarray}
where $k_{a_{j}}(S,s^{*})$ is the vector of covariance between the test point and training points for the covariance function relating to the action $a_{j}\in\mathcal{A}$.

\section{Experiments}
\label{section:experiments}
\begin{figure}[!t]
\centering
\subfigure[GPIRL accuracy]{\includegraphics[width=1.65in]{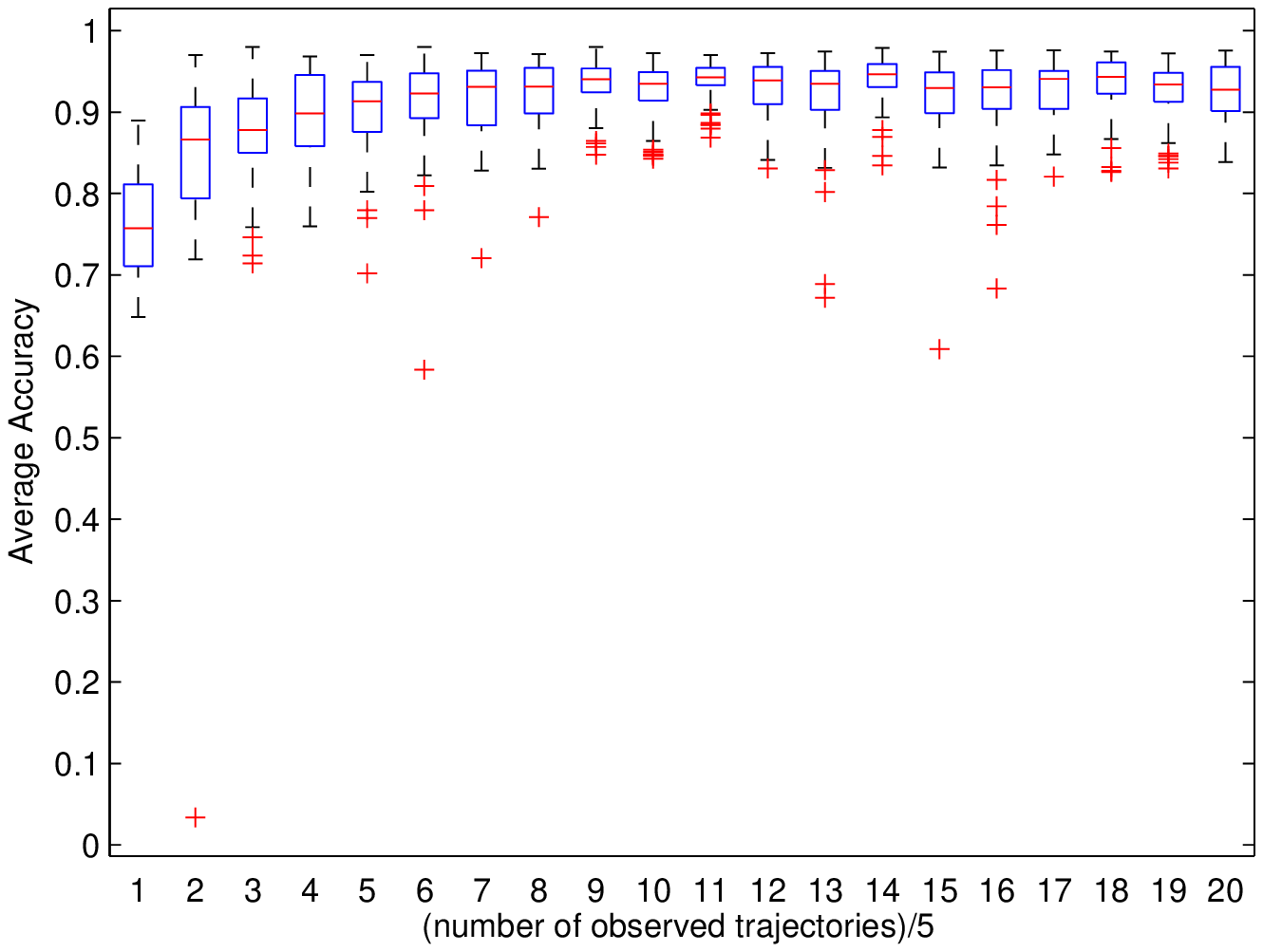}}
\hfill
\subfigure[LIRL accuracy]{\includegraphics[width=1.65in]{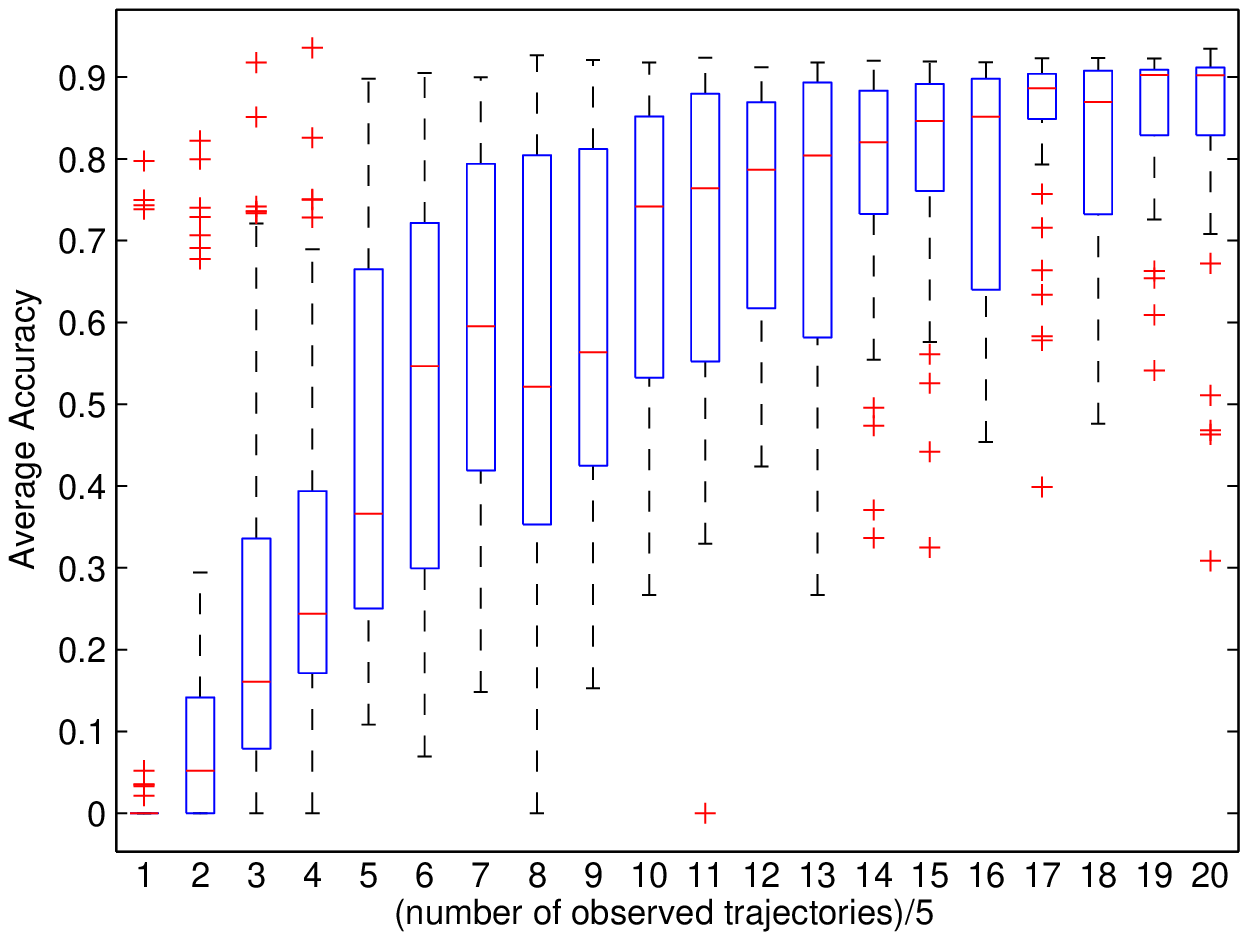}}
\caption{Average accuracy as a function of the number of observed decision trajectories, for GridWorld experiments.}
\label{fig_accAlongSamples}
\end{figure}
In this section, we report on a simple GridWorld experiment in which an agent starts from the a square of the grid and attempts to navigate to the goal square, with the possibility of encountering obstacles that block movement to certain squares. The agent is able to take five actions: remaining in the current square or moving in one of the four cardinal directions. Each movement action results in movement in the intended direction with probability 0.65, movement in an unintended direction with probability 0.2, and failure to move with probability 0.15. 

We compared three algorithms:  our convex programming method from Section \ref{section:convex} (\textsl{CPIRL}), our Gaussian process method from  Section \ref{section:gpmodel} (\textsl{GPIRL}), and the linear approximation method in \cite{andrew2000} (\textsl{LIRL}). Given observation of a complete policy, each of the algorithms was successful in finding a reward vector that yields an optimal policy identical to that observed. For each of the reward vectors returned by the algorithms, we recorded the amount of computation time needed  to find a best policy using reinforcement learning.  Table \ref{gridtime} shows the average of these time over 50 simulations.  Notably, reinforcement learning converges more quickly with reward vectors returned by CPIRL and GPIRL than with those returned by LIRL.  We hypothesize our methods tend to shape reward,  providing additional feedback to the agent and leading to an improvement in learning rate. 

Fig. \ref{fig_accAlongSamples} provides the basis for an accuracy comparison of \textsl{GPIRL} and \textsl{LIRL} for experiments in which only partial observations were available for reward learning.  Accuracy is calculated to be the fraction of runs in which the apprentice is able to achieve the teacher's goal state. The process of computing accuracy includes: 1) generating some GridWorld problems and sampling the decision trajectories from the teacher's demonstration; 2) inferring the reward function using GPIRL and LIRL;  3) generating 1000 new GridWorld problems with random initial state and solving these problems by applying reinforcement learning using the reward output by IRL;  4) comparing the results of the GPIRL and LIRL apprentices with the teacher.  If the apprentice reaches the teacher's goal state, we consider that trial a success for the apprentice.  As can be seen in Fig. \ref{fig_accAlongSamples}, the accuracy of GPIRL is higher than that of LIRL, especially when the number of observations is small.  Additionally,  GPIRL has clearly lower variance in accuracy. 
\begin{center}
\begin{threeparttable}[!t]
\caption{Time(sec) to find the apprentice policy}
\label{gridtimetable}
\begin{tabular}{llll}
\multicolumn{1}{c}{GridWorld Size}  &\multicolumn{1}{c}{LIRL} &\multicolumn{1}{c}{CPIRL} &\multicolumn{1}{c}{GPIRL}
\\ \hline \\
10x10         & 2.61 & 2.06 & 1.20\\
20x20					&	20.05		&	15.75		&9.32\\
30x30					& 75.12    &  64.30   &35.11\\
\hline
\end{tabular}
\label{gridtime}
\end{threeparttable}
\end{center}
\begin{figure}[!t]
\centering
\subfigure[60-state discretization]{\includegraphics[width=1.6in]{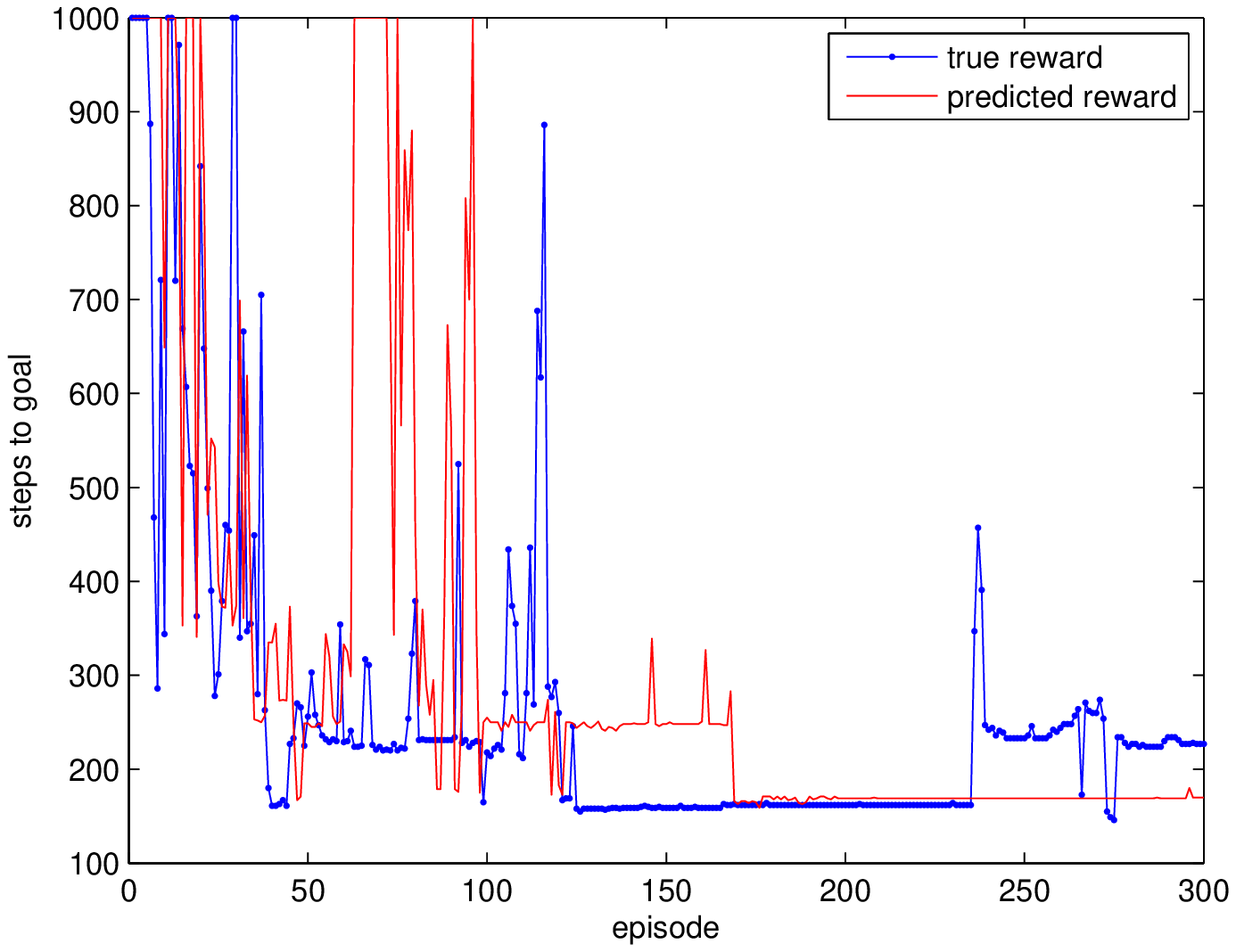}}
\subfigure[120-state discretization]{\includegraphics[width=1.6in]{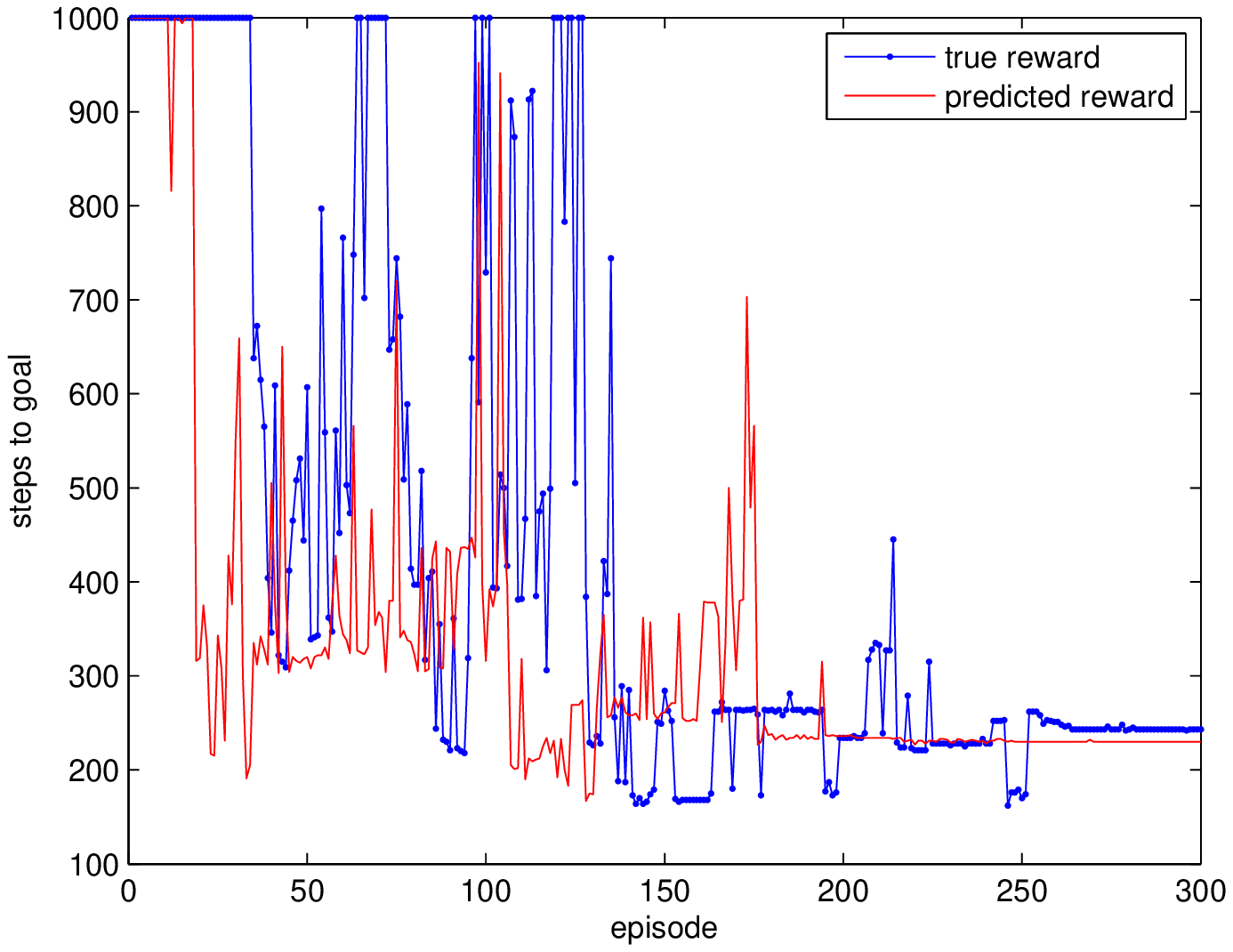}}
\caption{Solutions to the hill climbing problem based off true reward (blue) and reward recovered from GPIRL (red), for two levels of discretization.}
\label{figure:mntcar}
\end{figure}

We also performed an experiment based on a  simulation of an under-powered car attempting to drive out of a U-shaped valley. In this simulation, the car lacks enough power to climb the valley slopes from a standstill. Instead, it must first reverse up a slope in order to accumulate energy that will help it rush up the opposite slope. We choose the car's position and velocity as state features, discretizing those naturally continuous quantities.  To test GPIRL's ability to predict the reward on unseen states, we sampled only half the  discretized states as the observation data for GPIRL.  Given a state space with 120 states, for example, we would observe behavior in only 60 states.  Figure \ref{figure:mntcar} shows the number of steps needed to escape the valley for a range of starting conditions, or episodes, for policies learned from the true reward (blue) and from the reward returned by GPIRL (red). The results in the figure suggest that GPIRL is able to effectively recover the reward with incomplete observations, since the solver, using the reward predicted by GPIRL, has a  performance on par with that of the teacher, using true reward.

\section{Conclusions}
\label{section:conclusion}
We propose new IRL algorithms in the domain of convex programming. To deal with the IRL problems with ill-posed nature in large (or even infinite) state space, we model the reward using Gaussian process and interpret the observation of state-action space using preference graphs. Our posterior prediction method can estimate the reward at unobserved new coming states, which is promising for problems with large state space. Numerical experiments suggest that our method is able to find the reward approaching the true underlying reward with fewer observations than are needed with standard approaches. 
We will continue our research on IRL with Gaussian process in continuous space.
\section*{Acknowledgment}
This material is based upon work supported by the National Science Foundation under Grant No. EEC-0827153. 
\bibliographystyle{unsrt}
\bibliography{rls}
\end{document}